\documentclass[onecolumn]{nestpaper}
\usepackage{nicefrac}
\usepackage{amsfonts}
\usepackage{algorithm}
\usepackage{algpseudocodex}
\usepackage{textcomp}
\usepackage{import} 
\usepackage{color} 
\usepackage{overpic} 
\usepackage{siunitx} 

\newtheorem{definition}{Definition}
\newtheorem{problem}{Problem}
\newtheorem{theorem}{Theorem}
\newtheorem{lemma}{Lemma}
\newtheorem{corollary}{Corollary}
\newtheorem{remark}{Remark}

\let\cite\citep

\newcommand{\wdf}{\textbf{F}}
\newcommand{\wdfl}{\textbf{FL}}
\newcommand{\wdrl}{\textbf{RL}}
\newcommand{\wdr}{\textbf{R}}
\newcommand{\wdrr}{\textbf{RR}}
\newcommand{\wdfr}{\textbf{FR}}

\newcommand{\etal}{\textit{et al.}}
\newcommand{\fig}{Fig.\,}

\author{
  Khai Yi Chin\\
  Amazon Robotics \\
  \email{khaiyic@amazon.com}
  \and
  Tingwei Meng\\
  Amazon Robotics \\
  \email{tingweim@amazon.com}
  \and
  Zhe Chen\\
  Amazon Robotics \\
  \email{zhecm@amazon.com}
  \and
  Daniel Bassett \\
  Amazon Robotics \\
  \email{basjdani@amazon.com}
  \and
  Yuri Ivanov \\
  Amazon Robotics \\
  \email{yuriivan@amazon.com}}

\begin{document}

\title{Huddle: Parallel Shape Assembly using Decentralized, Minimalistic Robots}

\submitted{DARS 2026}
\date{}

\maketitle
\begin{abstract}
We propose a novel algorithm for forming arbitrarily shaped assemblies using decentralized robots. By relying on local interactions, the algorithm ensures there are no unreachable states or gaps in the assembly, which are global properties. The in-assembly robots attract passing-by robots into expanding the assembly via a simple implementation of signaling and alignment. Our approach is minimalistic, requiring only communication between attached, immediate neighbors. It is motion-agnostic and requires no pose localization, enabling asynchronous and order-independent assembly. We prove the algorithm's correctness and demonstrate its effectiveness in forming a 107-robot assembly.
\end{abstract}

\section{Introduction}\label{sec:introduction}

Multi-robot systems are often a logical alternative to single-robot systems in a variety of applications. Examples include transporting arbitrarily shaped objects~\cite{tuciCooperativeObjectTransport2018} and assembling structures on demand~\cite{oharaSelfassemblySwarmAutonomous2014}. In such cases, a multi-robot system offers speed, scalability, and robustness in its parallelized execution. This is further enhanced when the system is decentralized, as its centralized counterpart is associated with higher communication and computational burden from synchronizing robot activities. However, local robot interactions present a significant challenge, as they limit the range of conceivable emergent behaviors. The onus thus falls on designing correct local behaviors that lead to the desired global outcome.

In this work, we explore the self-assembly of decentralized, minimalistic robots into arbitrary 2-D shapes. To that end, we propose \textit{Huddle}, a parallel shape assembly algorithm. Huddle relies on simple signaling to enable homogeneous robots to self-assemble without specialized motion control strategies or pose localization. Starting with a single robot, the assembly gradually takes shape as in-assembly robots signal at specific openings to attract passing-by robots. Attachments to the assembly are anonymous and order-independent. Inter-robot communication is required only between in-assembly robots that are immediate neighbors.

The key novelty of Huddle is its minimalism, owing in part to its rule-based formulation, but also to its agnosticism towards motion control strategies. Inter-robot information exchange is equally modest, requiring only two integers shared at each time step. Moreover, Huddle only needs the perimeter of the desired shape to work, as we prove in our theoretical analysis. We also demonstrate the minimalism of Huddle using a physics-based simulated experiment where 107 robots assemble into a highly non-convex shape.

\section{Related Work}\label{sec:related_work}

Research on multi-robot self-assembly ranges from hardware design to algorithmic development~\cite{brayRecentDevelopmentsSelfAssembling2023}. Our work contributes to the algorithmic side: Huddle is a motion-agnostic, rule-based algorithm that enables robots to self-assemble using local interactions, providing global shape guarantees without localization requirements.

Due to their minimalism, rule-based methods are common in assembly planning. Notable studies include the self-assembly of Kilobots (hop count propagation with connectivity constraints)~\cite{rubensteinProgrammableSelfassemblyThousandrobot2014,gauciProgrammableSelfdisassemblyShape2018} and s-bots (hand-designed pattern extension rules)~\cite{grossAutonomousSelfAssemblySwarmBots2006,ogradySWARMORPHMultirobotMorphogenesis2009}, as well as the reconfiguration of SMORES-EP robots (centralized graph matching)~\cite{liuDistributedReconfigurationPlanning2019,liuSMORESEPModularRobot2023}. Salda\~na \etal{} used an assembly tree to guide square robots into specific docking sequences and locations~\cite{saldanaDecentralizedAlgorithmAssembling2017}. Using \textit{assembler} robots, Werfel \etal{} leveraged stigmergy through construction blocks that store and communicate assembly information~\cite{werfelDistributedConstructionMobile2006}. With 3-D self-assembly blocks, Stoy grew structures to match CAD models using a cellular automaton~\cite{stoyUsingCellularAutomata2006}, while Feshbach and Sung proposed a sequential method for assemblies to reshape themselves~\cite{feshbachReconfiguringNonConvexHoles2021}. Tolley and Lipson sampled attachment sites approved by a centralized algorithm to address attachment stochasticity~\cite{tolleyOnlineAssemblyPlanning2011}. Besides sharing features with such works---\textit{e.g.,} promotion of attachment sites, signaling and self-alignment mechanisms---Huddle offers decentralized, parallel shape assembly without connectivity, motion, or identification requirements.


Another popular approach employs optimization techniques. Matthey \etal{} synthesized robot controllers using the Chemical Reaction Network framework for the assembly of heterogeneous parts~\cite{mattheyStochasticStrategiesSwarm2009}. Sun \etal{} used the mean-shift algorithm in their shape assembly strategy, which optimizes density functions online~\cite{sunMeanshiftExplorationShape2023}. Methods built upon artificial potential fields are widespread as well, particularly in swarm robotics. This includes assembly of robots based on specifications from 2-D point clouds~\cite{liDecentralizedProgressiveShape2019} or from analytic functions~\cite{pinciroliSelfOrganizingScalableShape2008,yangParallelShapeFormation2022}. Our approach offers simplicity both in computational demands and in shape specification, requiring only perimeter coordinates.
\section{Problem Statement}\label{sec:problem_statement}
This work addresses the question: \textit{How do we create an arbitrarily shaped, hole-free assembly using decentralized, minimalistic robots without specific motion control strategies?} By being motion-agnostic, we can be adaptable to application constraints. For example, centralized waypoint computation may be feasible in a closed warehouse environment, but less so in a large, outdoor setting where decentralized methods scale better.

Our robot has a hexagonal shape, permitting up to $6$ neighbor attachments, with a signal emitter-receiver pair on each wall (hexagon side) for signaling and communication. The signaling medium is kept abstract to maintain generality; practical implementations could involve optical or sonic devices. Each wall also has a passive attachment mechanism (\textit{e.g.,} magnets~\cite{saldanaDecentralizedAlgorithmAssembling2017}), a proximity sensor for obstacle avoidance, and the ability to identify wall occupancy.

Our approach is to use in-assembly robots to attract passing-by robots, starting with a root robot. The finalized assembly must match the target shape $S$ exactly---provided that the shape is decomposable into hexagonal components. We assume passing-by robots can detect and follow the signals transmitted by in-assembly robots.

\begin{figure}
    \centering
    \begin{subfigure}{0.35\columnwidth}
        \centering
        \def\svgwidth{0.9\textwidth}
        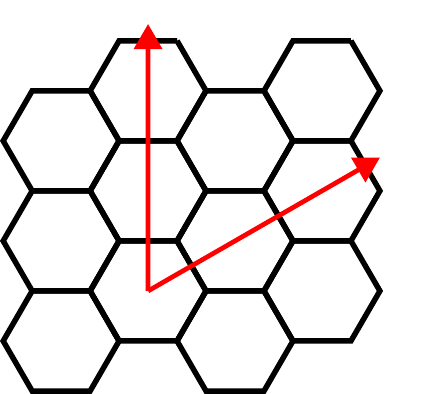
        \subcaption{}\label{fig:hex_layout}
    \end{subfigure}
    \hfill
    \begin{subfigure}{0.15\columnwidth}
        \centering
        \def\svgwidth{0.9\textwidth}
        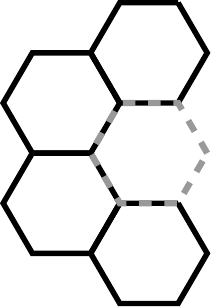
        \subcaption{}\label{fig:unreachable}
    \end{subfigure}
    \hfill
    \begin{subfigure}{0.45\columnwidth}
        \centering
        \def\svgwidth{0.6\textwidth}
        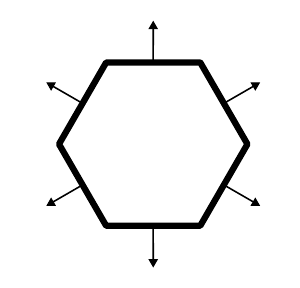
        \subcaption{}\label{fig:single_hex}
    \end{subfigure}
    \caption{(a) Axial coordinate system for the 2-D hexagonal space $\mathcal{H}$ that Huddle uses, where $p$ and $q$ are the column and row coordinates, respectively. (b) An unreachable position shown by the dashed hexagon. (c) Layout of wall indices and their corresponding directional unit vectors $\hat{u}_w$.}
\end{figure}

Formally, we represent a robot assembly at time $t$ using an undirected, planar graph, $G(t) = (V(t),E(t))$. Each in-assembly robot $i$ is represented by vertex $v_i \in V(t)$ and its attachment to neighbor $j$ is represented by edge $e_{ij} = e_{ji} \in E(t)$. Without loss of generality, we use the axial convention for the 2-D hexagonal coordinate frame $\mathcal{H} = \mathbb{Z}^2$ (\fig{}\ref{fig:hex_layout}). Each in-assembly robot occupies position $x_i \coloneqq (p_i,q_i) \in X(t) \subseteq \mathcal{H}$, where $p_i$ and $q_i$ are the column and row coordinates of robot $i$, respectively, and $X(t)$ is the set of occupied positions in $G(t)$. Only robots that share an edge can communicate.

To build $G(t)$, we assume an input of the target shape, given as a set of hexagonal coordinates ($S \subseteq \mathcal{H}$) that does not contain holes. As such, $S$ can be summarily represented by its perimeter $X_{b}$, a set of boundary coordinates---coordinates of hexagons in $S$ that have at least one unattached side. We assume that each robot has access to this set, which is invariant and can be provided offline or locally propagated from the root robot.

Passing-by robots attach only to an opening in the assembly that is formed by at least one (connectivity) and at most three robots (feasibility). We deem such an opening to be \textit{reachable}. Fig.\,\ref{fig:unreachable} shows an \emph{unreachable} opening.
\begin{definition}[Reachability]\label{def:reachability}
    Let $\mathcal{N}(x)$ be the adjacent coordinates of position $x$ and $\deg(v)$ be the degree of vertex $v$. At $t$, we have $Y(t) = \{y \in S \backslash X(t): \mathcal{N}(y) \cap X(t) \neq \emptyset \}$ as the set of unoccupied positions adjacent to $X(t)$. A position $y \in Y(t)$ is \textit{reachable} if adding a vertex $v_j$ with coordinates $x_j = y$ would yield $1 \leq \deg(v_j)\leq 3$. $G(t)$ \textit{maintains reachability} if all positions in $Y(t)$ are reachable.
\end{definition}

For the target shape to form without holes, the assembly must not contain enclosed empty regions at any time.
\begin{definition}[Hole-Free Assembly]\label{def:hole-free}
    $G(t)$ is \textit{hole-free} if no simple cycle $C$ in $G(t)$ bounds an interior region containing unoccupied positions from $S$.
\end{definition}

Let $t_{f}$ be the time $S$ is completed by $G(t)$. The problem we address in this work is as follows.
\begin{problem}\label{prob:problem_statement}
Given a hole-free shape $S$ represented by perimeter coordinates $X_{b}$, determine a decentralized method for signal activation such that:
\begin{enumerate}
    \item $G(t)$ maintains reachability $\forall t < t_f$ (Definition~\ref{def:reachability}), and
    \item $G(t)$ is hole-free $\forall t$ (Definition~\ref{def:hole-free}).
\end{enumerate}
\end{problem}
\section{Methodology}\label{sec:methodology}
We present Huddle, a parallel and decentralized shape assembly algorithm. The assembly grows in a concurrent, staggered manner, prioritizing longitudinal growth (intra-column) without blocking lateral expansion (inter-column). A designated \emph{root} robot initiates the assembly by remaining stationary at a target location. As the first in-assembly robot, it attracts passing-by robots to join by beaming signals; newly joined robots in turn emit their own signals, progressively expanding the assembly until $S$ is achieved. This growth pattern induces a \emph{spanning tree} over the assembly's column segments, where \emph{each segment is a node} and \emph{edges connect segments of adjacent columns}.

Huddle accomplishes this via two phases, \textit{Initialization} and \textit{Signal Activation} (\fig{}\ref{fig:high_level_assembly}). The \textit{Initialization} phase occurs once, at the instant a passing-by robot docks into the assembly. Then, this robot enters the \textit{Signal Activation} phase at each time step to attract other passing-by robots.


\begin{figure*}
    \centering
    \begin{subfigure}{0.193\textwidth}
        \centering
        \def\svgwidth{\textwidth}
\newcommand{\robotcoordfontsize}{6.5pt}
\newcommand{\wallstatusfontsize}{5pt}
\newcommand{\wallstatuslegendfontsize}{5pt}
\begingroup%
\makeatletter%
\providecommand\color[2][]{%
  \errmessage{(Inkscape) Color is used for the text in Inkscape, but the package 'color.sty' is not loaded}%
  \renewcommand\color[2][]{}%
}%
\providecommand\transparent[1]{%
  \errmessage{(Inkscape) Transparency is used (non-zero) for the text in Inkscape, but the package 'transparent.sty' is not loaded}%
  \renewcommand\transparent[1]{}%
}%
\providecommand\rotatebox[2]{#2}%
\newcommand*\fsize{\dimexpr\f@size pt\relax}%
\newcommand*\lineheight[1]{\fontsize{\fsize}{#1\fsize}\selectfont}%
\ifx\svgwidth\undefined%
  \setlength{\unitlength}{479.99998558bp}%
  \ifx\svgscale\undefined%
    \relax%
  \else%
    \setlength{\unitlength}{\unitlength * \real{\svgscale}}%
  \fi%
\else%
  \setlength{\unitlength}{\svgwidth}%
\fi%
\global\let\svgwidth\undefined%
\global\let\svgscale\undefined%
\makeatother%
\begin{picture}(1,0.87306258)%
  \lineheight{1}%
  \setlength\tabcolsep{0pt}%
  \put(0,0){\includegraphics[width=\unitlength,page=1]{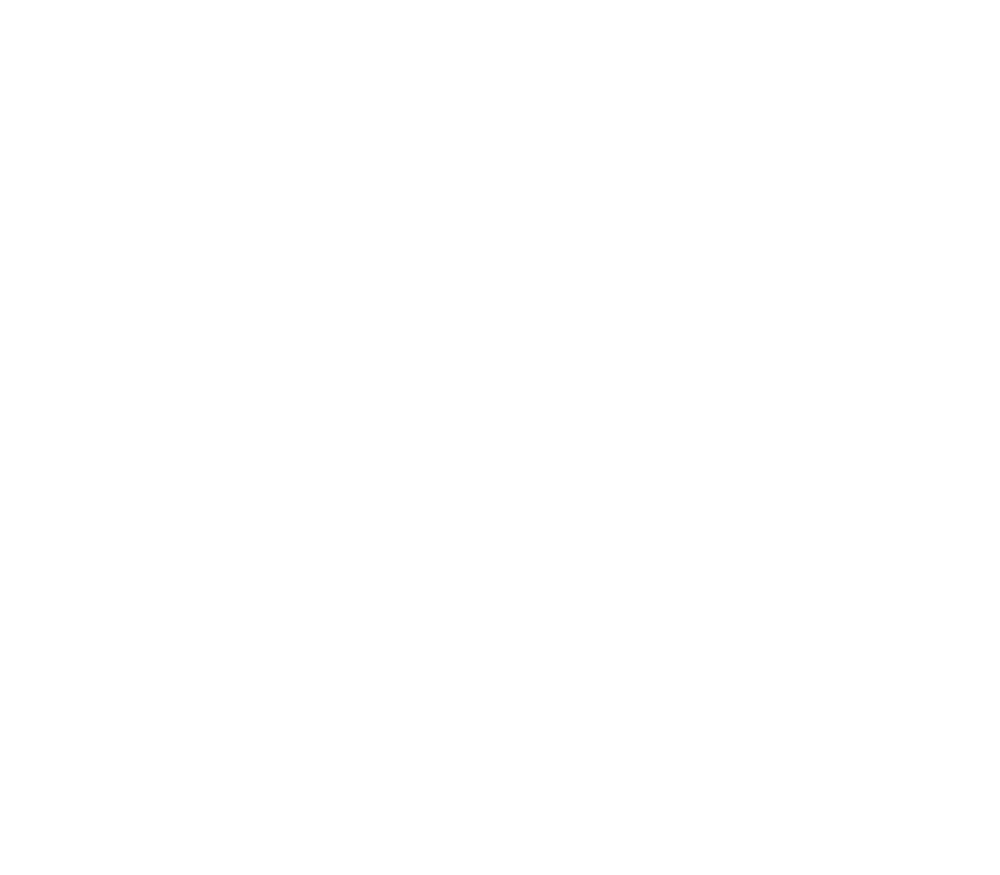}}%
  \put(0,0){\includegraphics[width=\unitlength,page=2]{assembly_1.pdf}}%
  \put(0.50000001,0.15289079){\color[rgb]{0,0,0}\makebox(0,0)[t]{\fontsize{\robotcoordfontsize}{0pt}\selectfont\lineheight{1.25}\smash{\begin{tabular}[t]{c}$(0,0)$\end{tabular}}}}%
\end{picture}%
\endgroup%

        \subcaption{}
    \end{subfigure}
    \hfill
    \begin{subfigure}{0.193\textwidth}
        \centering
        \def\svgwidth{\textwidth}
\newcommand{\robotcoordfontsize}{6.5pt}
\newcommand{\wallstatusfontsize}{5pt}
\newcommand{\wallstatuslegendfontsize}{5pt}
\begingroup%
\makeatletter%
\providecommand\color[2][]{%
  \errmessage{(Inkscape) Color is used for the text in Inkscape, but the package 'color.sty' is not loaded}%
  \renewcommand\color[2][]{}%
}%
\providecommand\transparent[1]{%
  \errmessage{(Inkscape) Transparency is used (non-zero) for the text in Inkscape, but the package 'transparent.sty' is not loaded}%
  \renewcommand\transparent[1]{}%
}%
\providecommand\rotatebox[2]{#2}%
\newcommand*\fsize{\dimexpr\f@size pt\relax}%
\newcommand*\lineheight[1]{\fontsize{\fsize}{#1\fsize}\selectfont}%
\ifx\svgwidth\undefined%
  \setlength{\unitlength}{482.27333369bp}%
  \ifx\svgscale\undefined%
    \relax%
  \else%
    \setlength{\unitlength}{\unitlength * \real{\svgscale}}%
  \fi%
\else%
  \setlength{\unitlength}{\svgwidth}%
\fi%
\global\let\svgwidth\undefined%
\global\let\svgscale\undefined%
\makeatother%
\begin{picture}(1,0.86894712)%
  \lineheight{1}%
  \setlength\tabcolsep{0pt}%
  \put(0,0){\includegraphics[width=\unitlength,page=1]{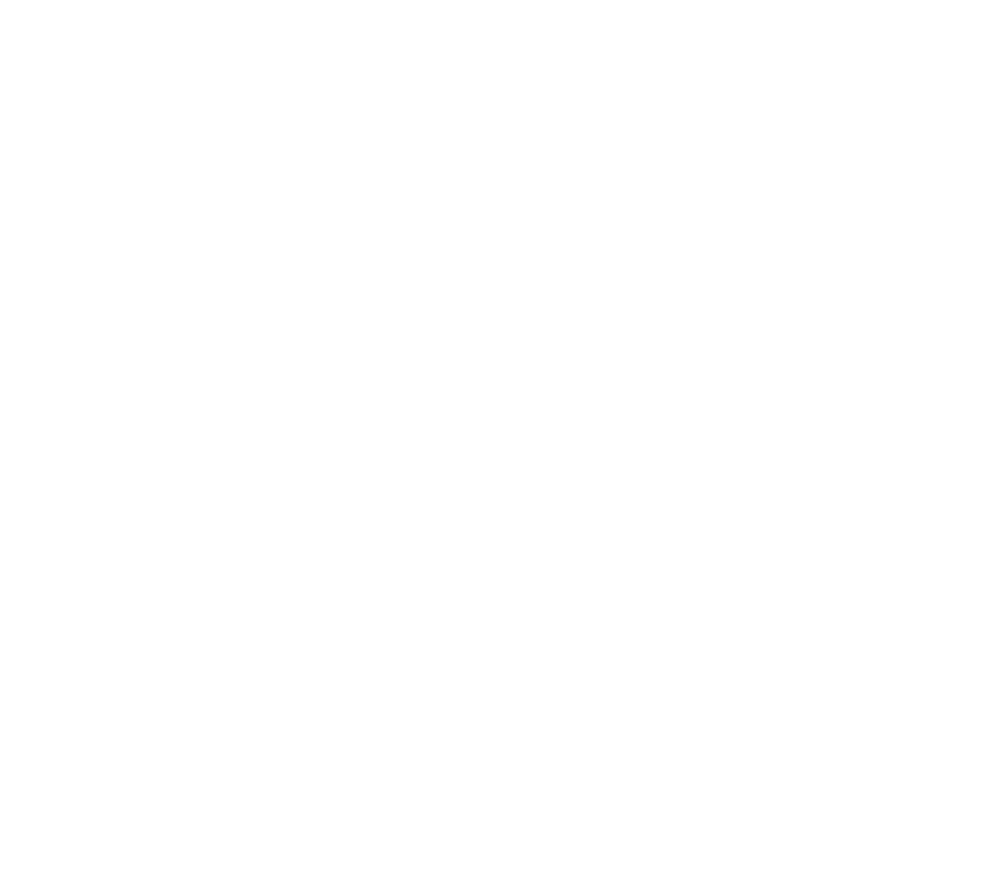}}%
  \put(0,0){\includegraphics[width=\unitlength,page=2]{assembly_2.pdf}}%
  \put(0.50033668,0.15217009){\color[rgb]{0,0,0}\makebox(0,0)[t]{\fontsize{\robotcoordfontsize}{0pt}\selectfont\lineheight{1.25}\smash{\begin{tabular}[t]{c}$(0,0)$\end{tabular}}}}%
\end{picture}%
\endgroup%

        \subcaption{}
    \end{subfigure}
    \hfill
    \begin{subfigure}{0.193\textwidth}
        \centering
        \def\svgwidth{\textwidth}
        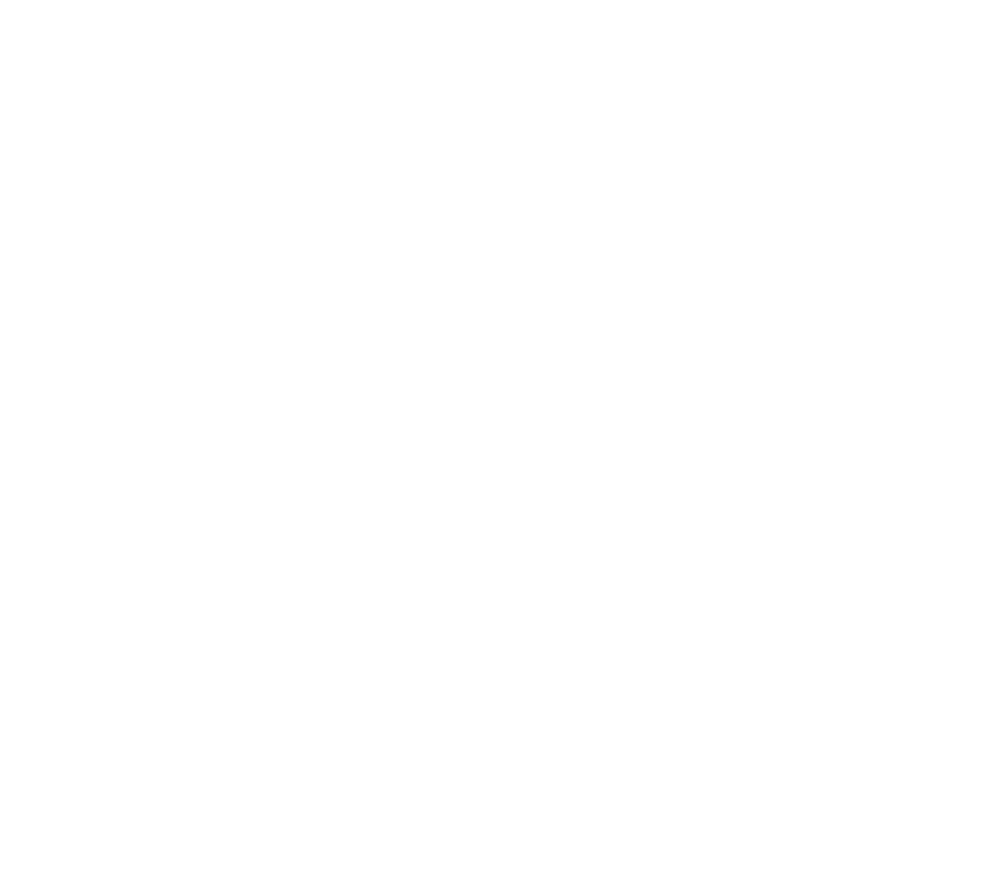
        \subcaption{}\label{fig:high_level_assembly_3}
    \end{subfigure}
    \hfill
    \begin{subfigure}{0.193\textwidth}
        \centering
        \def\svgwidth{\textwidth}
        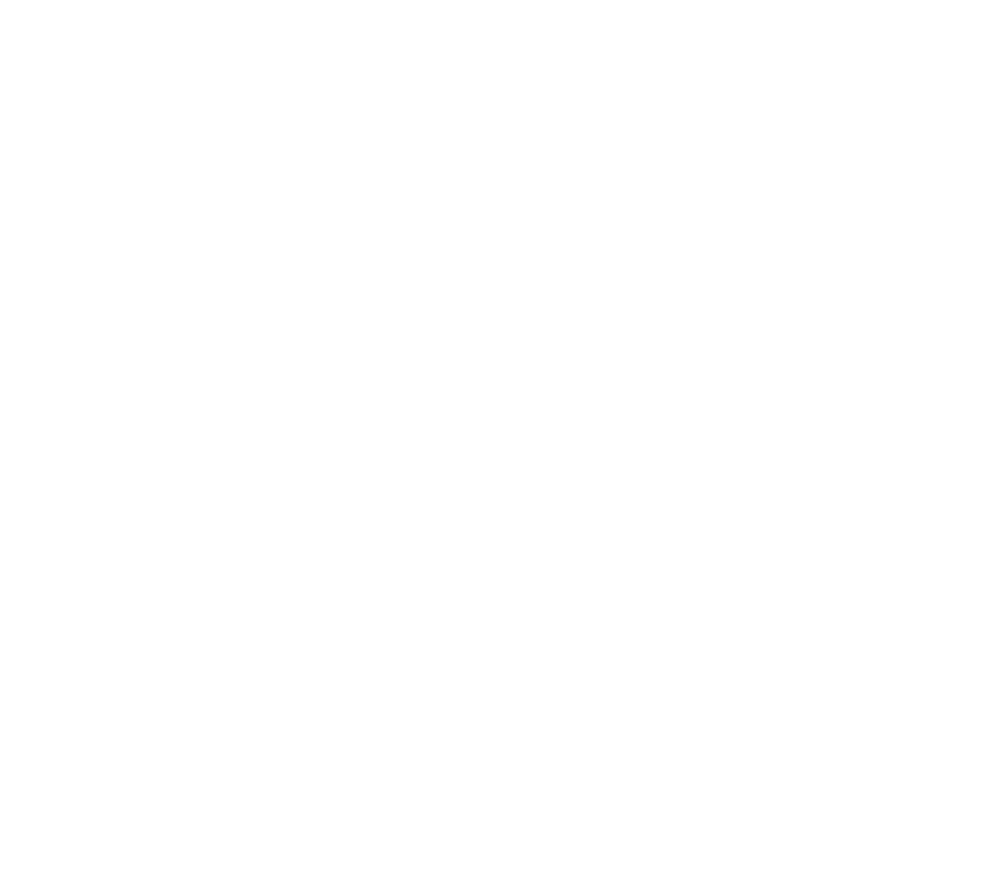
        \subcaption{}
    \end{subfigure}
    \hfill
    \begin{subfigure}{0.193\textwidth}
        \centering
        \def\svgwidth{\textwidth}
\newcommand{\robotcoordfontsize}{6.5pt}
\newcommand{\wallstatusfontsize}{5pt}
\newcommand{\wallstatuslegendfontsize}{5pt}
\begingroup%
\makeatletter%
\providecommand\color[2][]{%
  \errmessage{(Inkscape) Color is used for the text in Inkscape, but the package 'color.sty' is not loaded}%
  \renewcommand\color[2][]{}%
}%
\providecommand\transparent[1]{%
  \errmessage{(Inkscape) Transparency is used (non-zero) for the text in Inkscape, but the package 'transparent.sty' is not loaded}%
  \renewcommand\transparent[1]{}%
}%
\providecommand\rotatebox[2]{#2}%
\newcommand*\fsize{\dimexpr\f@size pt\relax}%
\newcommand*\lineheight[1]{\fontsize{\fsize}{#1\fsize}\selectfont}%
\ifx\svgwidth\undefined%
  \setlength{\unitlength}{479.99998558bp}%
  \ifx\svgscale\undefined%
    \relax%
  \else%
    \setlength{\unitlength}{\unitlength * \real{\svgscale}}%
  \fi%
\else%
  \setlength{\unitlength}{\svgwidth}%
\fi%
\global\let\svgwidth\undefined%
\global\let\svgscale\undefined%
\makeatother%
\begin{picture}(1,0.87306258)%
  \lineheight{1}%
  \setlength\tabcolsep{0pt}%
  \put(0,0){\includegraphics[width=\unitlength,page=1]{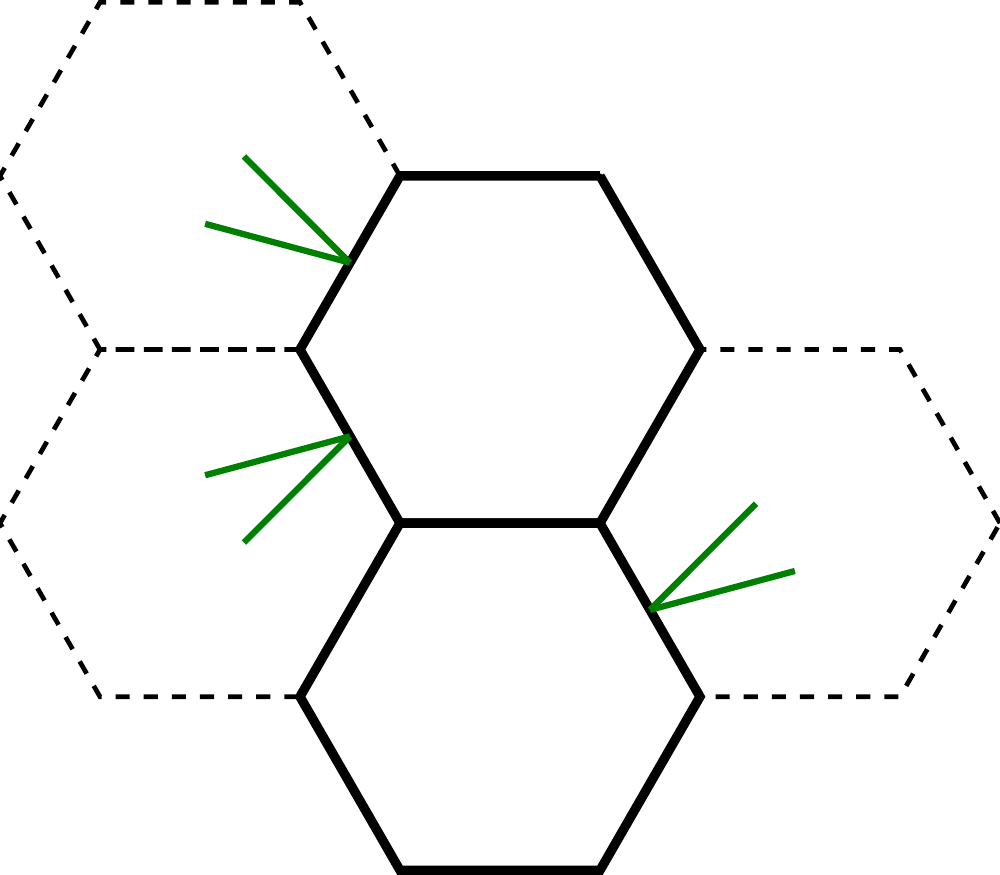}}%
  \put(0,0){\includegraphics[width=\unitlength,page=2]{assembly_5.pdf}}%
  \put(0.50000001,0.15289079){\color[rgb]{0,0,0}\makebox(0,0)[t]{\fontsize{\robotcoordfontsize}{0pt}\selectfont\lineheight{1.25}\smash{\begin{tabular}[t]{c}$(0,0)$\end{tabular}}}}%
  \put(0.50000001,0.49930486){\color[rgb]{0,0,0}\makebox(0,0)[t]{\fontsize{\robotcoordfontsize}{0pt}\selectfont\lineheight{1.25}\smash{\begin{tabular}[t]{c}$(0,1)$\end{tabular}}}}%
\end{picture}%
\endgroup%

        \subcaption{}
    \end{subfigure}
    \caption{Example of a 5-robot assembly. Dashed hexagons are target positions in $S$. (a) A root robot $(0,0)$ in its \textit{Signal Activation} phase, signaling (green) on its \wdf{} wall to attract passing-by robots. (b) A passing-by robot detects and follows the signal to attach to the root robot. Upon connection, the root robot communicates its position $(0,0)$ and the unit vector of the admitting wall $\hat{u}_0$ to the newly attached robot. As a right nucleus---determined prior to its entering the \textit{Signal Activation} phase---the root robot now signals on its \wdfr{} wall. (c) The newly attached robot enters the \textit{Initialization} phase, computing its coordinate $(0,1)$ and identifying its wall statuses. (d) The $(0,1)$ robot determines its role as a left nucleus and its growth direction as $g_i = -1$, as it shares a row with the (rounded down) midpoint $(-1,1)$ of the left column (Algorithms~\ref{alg:determine_role}--\ref{alg:determine_growth_direction}). (e) The $(0,1)$ robot enters the \textit{Signal Activation} phase, signaling its left flank (Algorithms~\ref{alg:identify_nominal_walls} and \ref{alg:delay_signal_activation}). The robots remain in that phase until their connection walls are occupied.} \label{fig:high_level_assembly}

\end{figure*}

\subsection{Preliminaries}\label{sec:methodology_preliminaries}

\subsubsection{Wall type}
A robot wall can be one of two types.
\begin{itemize}
    \item Null wall: A wall on an in-assembly robot \emph{not} intended for attachment.
    \item Connection wall: A wall on an in-assembly robot intended for attachment.
\end{itemize}
Each wall also has a \textit{status}: a null wall is always ``null'', while a connection wall can be ``free'' or ``occupied'' (\fig{}\ref{fig:high_level_assembly_3}).

\subsubsection{Wall direction}
As shown in \fig{}\ref{fig:single_hex}, we label the robot walls in a counter-clockwise manner
where $\hat{u}_w \in \mathcal{H}$ is the unit vector pointing in the direction of wall $w$ from the robot center. When we refer to ``left flank", walls \wdfl{} and \wdrl{} are included; likewise for ``right flank". We use the term ``fore-aft" when referring to the front (\wdf{}) and rear (\wdr{}) walls.

\subsubsection{Robot roles}\label{sec:methodology_preliminaries_roles}
Each robot has one of two fixed roles, determined during the \textit{Initialization} phase.
\begin{itemize}
    \item Nucleus robot: An in-assembly robot that \emph{seeds} the expansion of an adjacent column segment (\textit{i.e.,} spawns a child node in the spanning tree); it signals on at least one flank (one or both walls).
    \item Regular robot: An in-assembly robot that \emph{does not seed} any adjacent column segment; it signals on a flank wall only under a special condition.
\end{itemize}
We describe the special condition in Remark~\ref{rem:special_condition}. Crucially, nucleus robots' ability to seed adjacent columns drives the assembly's \emph{lateral expansion}, as we detail in Section~\ref{sec:methodology_signal}. Both roles can signal on \wdf{} and \wdr{} walls for \emph{longitudinal expansion}.

\subsection{Initialization}\label{sec:methodology_initialization}
When a passing-by robot $i$ attaches to the assembly, it processes information received from its newly connected, in-assembly neighbor $j$. This information consists of robot $j$'s position $x_j$, wall status vector $s_j(t)= (s_{j,0}(t), ..., s_{j,5}(t))$, where $s_{j,w}(t) \in \{\text{null}, \text{free}, \text{occupied}\}$, and the unit vector $\hat{u}_{j,w}$ of robot $j$'s admitting wall $w$. For brevity, we drop the robot index in the unit vector ($\hat{u}_w$) when clear from context, and omit time $t$ in pseudocode.

\subsubsection{Wall Type Identification}




Robot $i$ first determines its position $x_i = x_j + \hat{u}_{j,w}$, then identifies its wall types and statuses to update $s_i(t)$. To do this, robot $i$ finds the set of minimal cycles $X_c$ contained in $X_b$ using Johnson's algorithm~\cite{johnsonFindingAllElementary1975}. It then checks neighboring positions that are either part of $X_b$ or enclosed by a cycle in $X_c$ (using a ray-casting algorithm~\cite{heckbertGraphicsGems2013}). Walls facing these positions are connection walls; otherwise they are null walls.


\subsubsection{Role Determination}
\begin{algorithm}
    \caption{Determine robot role}\label{alg:determine_role}
    \begin{algorithmic}[1]
        \State \textbf{Input:} Self-wall status vector $s_i$, perimeter $X_b$, cycles $X_c$
        \State \textbf{Output:} Nucleation requirement $\alpha_{\text{left}}, \alpha_{\text{right}}$

        \State Initialize $\tilde{\alpha}_{\text{left}} \gets$ undefined, $\tilde{\alpha}_{\text{right}} \gets$ undefined

        \State Evaluate $\tilde{\alpha}_{\text{left}}, \tilde{\alpha}_{\text{right}}$ with Algorithm~\ref{alg:determine_role_using_self}

        \If{$\tilde{\alpha}_{\text{left}}=$ undefined}
        \State Evaluate $\tilde{\alpha}_{\text{left}}$ with Algorithm~\ref{alg:determine_role_using_adjacent_col}
        \EndIf

        \If{$\tilde{\alpha}_{\text{right}}=$ undefined}
        \State Evaluate $\tilde{\alpha}_{\text{right}}$ with Algorithm~\ref{alg:determine_role_using_adjacent_col}
        \EndIf

        \State $\alpha_{\text{left}} \gets \tilde{\alpha}_{\text{left}}$, $\alpha_{\text{right}} \gets \tilde{\alpha}_{\text{right}}$

        \State \Return $\alpha_{\text{left}}, \alpha_{\text{right}}$ \Comment{$\alpha_{\text{left}}, \alpha_{\text{right}} \in \{\text{false}, \text{true}\}$}
    \end{algorithmic}
\end{algorithm}
Next, robot $i$ evaluates its role using Algorithm~\ref{alg:determine_role}, which relies on two subroutines, Algorithm~\ref{alg:determine_role_using_self} and Algorithm~\ref{alg:determine_role_using_adjacent_col}. We infer robot $i$'s role from the output of Algorithm~\ref{alg:determine_role}, which are boolean flags indicating if nucleation is required for a flank.

\begin{algorithm}
    \caption{Determine robot role using self-wall status}\label{alg:determine_role_using_self}
    \begin{algorithmic}[1]
        \State \textbf{Input:} Self-wall status vector $s_i$
        \State \textbf{Output:} Nucleation requirement $\beta_{\text{left}}, \beta_{\text{right}}$

        \State Initialize $\beta_{\text{left}} \gets$ false, $\beta_{\text{right}} \gets$ false

        \For{$F \in \{\text{left}, \text{right}\}$}
        \If{$F$ flank has at least 1 free wall}
        \If{free $F$ flank wall is between 2 null walls}
        \State $\beta_F \gets \text{true}$
        \Else
        \State $\beta_F \gets \text{undefined}$
        \EndIf
        \EndIf
        \EndFor



        \State \Return $\beta_{\text{left}}, \beta_{\text{right}}$ \Comment{$\beta_{\text{left}}, \beta_{\text{right}} \in \{\text{false}, \text{true}, \text{undefined}\}$}
    \end{algorithmic}
\end{algorithm}

The first subroutine, Algorithm~\ref{alg:determine_role_using_self}, considers only robot $i$'s wall status vector $s_i$ in deciding nucleation requirements. It checks for connection flank walls that are situated between two null walls; if such walls exist, robot $i$ nucleates on the flanks with said wall arrangement.

However, $s_i(t)$ alone may be insufficient for identifying nucleation requirements. Thus, a second subroutine, Algorithm~\ref{alg:determine_role_using_adjacent_col}, does so by using adjacent target columns relative to robot $i$'s position $x_i = (p_i, q_i)$ as illustrated in \fig{}\ref{fig:algorithm_4_explanation}:
\begin{align*}\label{eq:target_adjacent_column}
    A_{p_i+a} = \big\{ Q_{p_i+a}^{(1)}, ..., Q_{p_i+a}^{(N_{p_i+a})} \big\}
\end{align*}
where $a \in \{-1, 0, 1\}$ is the column direction (left, current, and right) and $p_i + a$ is the column index. $A_{p_i+a}$ is a target column containing $N_{p_i+a}$ contiguous segments, denoted $Q_{p_i + a}^{(n)}$. When only one segment (out of many) is being discussed, we drop the superscript ${}^{(n)}$ and identify it descriptively.

Algorithm~\ref{alg:determine_role_using_adjacent_col} checks whether $x_i$ is on the same row as midpoint $m$ of a target adjacent segment $Q_{p_i + a}$. If so, robot $i$ is designated as a nucleus. Otherwise, the algorithm verifies whether $m$ has a valid connection in $Q_{p_i}$, and if not, finds the next closest point $m^\star$---making robot $i$ the nucleus if $m^\star$ is on its row. Only one nucleus is responsible for seeding each adjacent segment (proved in Section~\ref{sec:theoretical_results}).


\begin{algorithm}
    \caption{Determine robot role using target adjacent column}\label{alg:determine_role_using_adjacent_col}
    \begin{algorithmic}[1]
        \State \textbf{Input:} Perimeter $X_b$, cycles $X_c$, self-coordinate $x_i$, adjacent direction $a \in \{-1,1\}$
        \State \textbf{Output:} Nucleation requirement $\gamma$

        \State Find target segments in $A_{p_i+a}$ and their midpoints $M$ using $X_b$, $X_c$, and $x_i$
        \If{$A_{p_i+a} = \emptyset$} \Comment{no target segments in $A_{p_i+a}$}
        \State $\gamma \gets \text{false}$
        \Else
        \State Identify target segment $Q_{p_i+a} \in A_{p_i+a}$ and its midpoint $m = (p_i + a, q_m)$
        \If{$q_i = q_m$} \Comment{$m$ is on the same row as $x_i$}
        \State $\gamma \gets \text{true}$
        \Else \Comment{$m$ is not on the same row as $x_i$}

        \State Find self-target segment $Q_{p_i}$ \Comment{target segment containing $x_i$}
        \If{$(p_i, q_m) \in Q_{p_i}$} \Comment{$m$ is on the same row as another point in $Q_{p_i}$}\label{alg:determine_role_using_adjacent_col_case_1_if}
        \State $\gamma \gets \text{false}$\label{alg:determine_role_using_adjacent_col_case_1_endif}
        \Else \Comment{$m$ does not have a valid attachment point in $Q_{p_i}$}

        \State Find $m^\star = (p_i+a,q^{\star}) \in Q_{p_i+a}$ nearest to $m$ that connects to $Q_{p_i}$
        \If{$ q_i = q^{\star}$} \Comment{$m^\star$ is on the same row as $x_i$}
        \State $\gamma \gets \text{true}$
        \Else \Comment{$m^\star$ is not on the same row as $x_i$}\label{alg:determine_role_using_adjacent_col_case_2b_else}
        \State $\gamma \gets \text{false}$\label{alg:determine_role_using_adjacent_col_case_2b_endelse}
        \EndIf
        \EndIf
        \EndIf
        \EndIf
        \State \Return $\gamma$ \Comment{$\gamma \in \{\text{false}, \text{true}\}$}
    \end{algorithmic}
\end{algorithm}

\subsubsection{Growth Direction Determination}
At this point, robot $i$ is either a nucleus robot (nucleating on the left-, right-, or both flanks) or a regular robot (non-nucleating). Subsequently, robot $i$ determines its growth direction, $g_i$ using Algorithm~\ref{alg:determine_growth_direction}. For a nucleus robot, $g_i$ indicates where it will initiate lateral expansion; for a regular robot, $g_i$ indicates the expansion direction it supports.

\begin{algorithm}
    \caption{Determine growth direction}\label{alg:determine_growth_direction}
    \begin{algorithmic}[1]
        \State \textbf{Input:} Nucleation requirement $\alpha_{\text{left}},\alpha_{\text{right}}$, self-wall status vector $s_i$
        \State \textbf{Output:} Growth direction, $g_i$

        \If{$\alpha_{\text{left}}$ \textbf{and} $\alpha_{\text{right}}$ = true}
        \State $g_i \gets 0$ \Comment{neutral (nucleus)}
        \ElsIf{($s_{i,1}$ \textbf{or} $s_{i,2}=$ occupied) \textbf{or} $\alpha_{\text{right}}=$ true}
        \State $g_i \gets 1$ \Comment{right}
        \ElsIf{($s_{i,4}$ \textbf{or} $s_{i,5}=$ occupied) \textbf{or} $\alpha_{\text{left}}=$ true}
        \State $g_i \gets -1$ \Comment{left}
        \Else
        \State $g_i \gets 0$ \Comment{neutral (regular)}
        \EndIf
        \State \Return $g_i$ \Comment{$g_i \in \{-1, 0, 1\}$}
    \end{algorithmic}
\end{algorithm}

By the end of the \textit{Initialization} phase, robot $i$ has acquired information about its connection walls, role, and growth direction. Note that, as a boundary condition for our approach, the root robot assumes the origin $x_i = (0,0)$ and skips the step where it receives information from a neighbor.

\subsection{Signal Activation}\label{sec:methodology_signal}
In this phase, robot $i$ determines which walls require signal activation. We assume the ``signal'' is a transmission of an IR beam. It need not contain any encoded information; we only require that a recipient robot can identify and follow it to the source. This phase runs iteratively until the robot's connection walls are occupied.\footnote{We do not consider robot rejoining here; if such a capability is desired, the \textit{Signal Activation} phase can operate perpetually with departure-handling mechanisms.} At the start of this phase, robot $i$ updates and communicates its wall status vector $s_i(t)$ with its neighbors.

\subsubsection{Signaling Wall Identification}
Then, using intrinsic information ($s_{i}(t)$, $\alpha_{\text{left}}$, $\alpha_{\text{right}}$), robot $i$ identifies nominal walls $W_i$ that require signal activation with Algorithm~\ref{alg:identify_nominal_walls}. These are connection walls found in the prior phase, with priority given to fore-aft walls. This ensures that each column is allowed to grow longitudinally (in the fore-aft directions) before nucleation happens laterally (in the flank directions), as premature lateral expansion could induce unreachable openings. Effectively, Algorithm~\ref{alg:identify_nominal_walls} ensures the occupancy of fore-aft walls first---if they are connection walls---before flank walls can signal for both robot roles, prioritizing each spanning tree node's internal growth before it branches.

\begin{algorithm}
    \caption{Identify walls to activate signals}\label{alg:identify_nominal_walls}
    \begin{algorithmic}[1]
        \State \textbf{Input:} Self-wall status vector $s_i$, nucleation requirement $\alpha_{\text{left}}$, $\alpha_{\text{right}}$
        \State \textbf{Output:} Set of walls $W_i$ requiring signals

        \State Initialize $W_i \gets \emptyset$

        \If{all $s_i =$ occupied} \Comment{no signaling needed}
        \State \Return $W_i$
        \EndIf

        \If{$s_{i,0}$ \textbf{or} $s_{i,3} =$ free} \Comment{prioritize fore-aft walls}\label{alg:identify_nominal_walls_case1_if}
        \For{fore-aft wall $w \in \{0,3\}$}
        \If{$s_{i,w} =$ free}
        \State $W_i \gets W_i \cup \{w\}$
        \EndIf
        \EndFor

        \State \Return $W_i$ \Comment{ignore flank walls because of free fore-aft walls}
        \EndIf\label{alg:identify_nominal_walls_case1_endif}

        \For{flank wall $w \in \{1, 2, 4, 5\}$ with $s_{i,w} = \text{free}$}\label{alg:identify_nominal_walls_case2_for} \Comment{special flank signal condition}
        \If{$w$ is between (2 occupied \textbf{or} (1 null fore-aft \textbf{and} 1 occupied)) walls}
        \State $W_i \gets W_i \cup \{w\}$
        \EndIf
        \EndFor\label{alg:identify_nominal_walls_case2_endfor}
        \If{$\alpha_{\text{left}} =$ true} \Comment{left nucleus flank walls}
        \State $W_i \gets W_i \cup \{\text{free left flank wall(s)}\}$
        \EndIf

        \If{$\alpha_{\text{right}} =$ true} \Comment{right nucleus flank walls}
        \State $W_i \gets W_i \cup \{\text{free right flank wall(s)}\}$
        \EndIf

        \State \Return $W_i$
    \end{algorithmic}
\end{algorithm}

\begin{remark}\label{rem:special_condition}
    The special condition mentioned in Section~\ref{sec:methodology_preliminaries_roles} refers to the logic in lines \ref{alg:identify_nominal_walls_case2_for}--\ref{alg:identify_nominal_walls_case2_endfor} of Algorithm~\ref{alg:identify_nominal_walls}. Added to improve passing-by robots' signal detection probability, it can be removed without affecting the assembly process.
\end{remark}

\begin{figure}
    \centering
    \begin{subfigure}{0.6\columnwidth}
        \centering
        \hspace{-30pt}
        \begin{subfigure}{\columnwidth}
            \centering
            \def\svgwidth{0.3\textwidth}
\newcommand{\labelfontsize}{8pt}
\newcommand{\labelhspace}{12pt}
\begingroup%
  \makeatletter%
  \providecommand\color[2][]{%
    \errmessage{(Inkscape) Color is used for the text in Inkscape, but the package 'color.sty' is not loaded}%
    \renewcommand\color[2][]{}%
  }%
  \providecommand\transparent[1]{%
    \errmessage{(Inkscape) Transparency is used (non-zero) for the text in Inkscape, but the package 'transparent.sty' is not loaded}%
    \renewcommand\transparent[1]{}%
  }%
  \providecommand\rotatebox[2]{#2}%
  \newcommand*\fsize{\dimexpr\f@size pt\relax}%
  \newcommand*\lineheight[1]{\fontsize{\fsize}{#1\fsize}\selectfont}%
  \ifx\svgwidth\undefined%
    \setlength{\unitlength}{366.08396035bp}%
    \ifx\svgscale\undefined%
      \relax%
    \else%
      \setlength{\unitlength}{\unitlength * \real{\svgscale}}%
    \fi%
  \else%
    \setlength{\unitlength}{\svgwidth}%
  \fi%
  \global\let\svgwidth\undefined%
  \global\let\svgscale\undefined%
  \makeatother%
  \begin{picture}(1,0.19053005)%
    \lineheight{1}%
    \setlength\tabcolsep{0pt}%
    \put(0,0){\includegraphics[width=\unitlength,page=1]{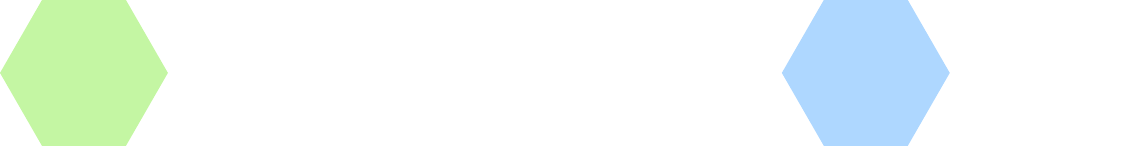}}%
    \put(0.25623849,0.05039795){\color[rgb]{0,0,0}\makebox(0,0)[t]{\fontsize{\labelfontsize}{0pt}\selectfont\lineheight{1.25}\smash{\begin{tabular}[t]{c}\hspace{\labelhspace}$\in Q^{(1)}_{-1}$\end{tabular}}}}%
    \put(0.95250885,0.05039795){\color[rgb]{0,0,0}\makebox(0,0)[t]{\fontsize{\labelfontsize}{0pt}\selectfont\lineheight{1.25}\smash{\begin{tabular}[t]{c}\hspace{\labelhspace}$\in Q^{(2)}_{-1}$\end{tabular}}}}%
  \end{picture}%
\endgroup%

        \end{subfigure}
        \par \smallskip
        \centering
        \begin{subfigure}{0.32\columnwidth}
            \centering
            \def\svgwidth{0.75\textwidth}
            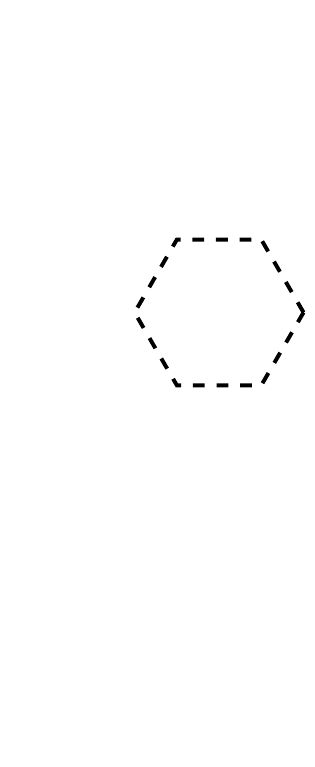
            \begin{overpic}[width=0.01\textwidth]{figures/blank.png}
                \put(0,0){\hspace{-68pt}\raisebox{110pt}{\small(i)}}  
            \end{overpic}
        \end{subfigure}
        \hfill
        \begin{subfigure}{0.32\columnwidth}
            \centering
            \def\svgwidth{0.75\textwidth}
            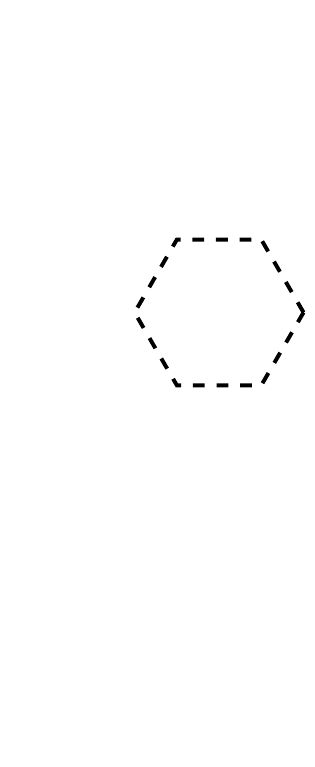
            \begin{overpic}[width=0.01\textwidth]{figures/blank.png}
                \put(0,0){\hspace{-68pt}\raisebox{110pt}{\small(ii)}}  
            \end{overpic}
        \end{subfigure}
        \hfill
        \begin{subfigure}{0.32\columnwidth}
            \centering
            \def\svgwidth{0.75\textwidth}
\newcommand{\hexcoordfontsize}{6pt}
\newcommand{\hexcoordvarfontsize}{10pt}
\begingroup%
  \makeatletter%
  \providecommand\color[2][]{%
    \errmessage{(Inkscape) Color is used for the text in Inkscape, but the package 'color.sty' is not loaded}%
    \renewcommand\color[2][]{}%
  }%
  \providecommand\transparent[1]{%
    \errmessage{(Inkscape) Transparency is used (non-zero) for the text in Inkscape, but the package 'transparent.sty' is not loaded}%
    \renewcommand\transparent[1]{}%
  }%
  \providecommand\rotatebox[2]{#2}%
  \newcommand*\fsize{\dimexpr\f@size pt\relax}%
  \newcommand*\lineheight[1]{\fontsize{\fsize}{#1\fsize}\selectfont}%
  \ifx\svgwidth\undefined%
    \setlength{\unitlength}{150.0000036bp}%
    \ifx\svgscale\undefined%
      \relax%
    \else%
      \setlength{\unitlength}{\unitlength * \real{\svgscale}}%
    \fi%
  \else%
    \setlength{\unitlength}{\svgwidth}%
  \fi%
  \global\let\svgwidth\undefined%
  \global\let\svgscale\undefined%
  \makeatother%
  \begin{picture}(1,2.44749952)%
    \lineheight{1}%
    \setlength\tabcolsep{0pt}%
    \put(0,0){\includegraphics[width=\unitlength,page=1]{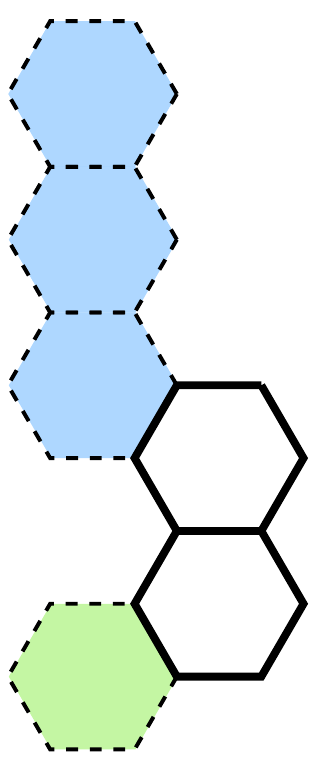}}%
    \put(0.69640772,0.46999976){\color[rgb]{0,0,0}\makebox(0,0)[t]{\fontsize{\hexcoordfontsize}{0pt}\selectfont\lineheight{1.25}\smash{\begin{tabular}[t]{c}$(0,0)$\end{tabular}}}}%
    \put(0.2964142,0.25749967){\color[rgb]{0,0,0}\makebox(0,0)[t]{\fontsize{\hexcoordfontsize}{0pt}\selectfont\lineheight{1.25}\smash{\begin{tabular}[t]{c}$(-1,0)$\end{tabular}}}}%
    \put(0.2964142,1.65249975){\color[rgb]{0,0,0}\makebox(0,0)[t]{\fontsize{\hexcoordfontsize}{0pt}\selectfont\lineheight{1.25}\smash{\begin{tabular}[t]{c}$(-1,3)$\end{tabular}}}}%
    \put(0.2964142,2.11749968){\color[rgb]{0,0,0}\makebox(0,0)[t]{\fontsize{\hexcoordfontsize}{0pt}\selectfont\lineheight{1.25}\smash{\begin{tabular}[t]{c}$(-1,4)$\end{tabular}}}}%
    \put(0.2964142,1.18749982){\color[rgb]{0,0,0}\makebox(0,0)[t]{\fontsize{\hexcoordfontsize}{0pt}\selectfont\lineheight{1.25}\smash{\begin{tabular}[t]{c}$(-1,2)$\end{tabular}}}}%
    \put(0.69628751,0.93499965){\color[rgb]{0,0,0}\makebox(0,0)[t]{\fontsize{\hexcoordvarfontsize}{0pt}\selectfont\lineheight{1.25}\smash{\begin{tabular}[t]{c}$i$\end{tabular}}}}%
  \end{picture}%
\endgroup%

            \begin{overpic}[width=0.01\textwidth]{figures/blank.png}
                \put(0,0){\hspace{-68pt}\raisebox{110pt}{\small(iii)}}  
            \end{overpic}
        \end{subfigure}
        \subcaption{}\label{fig:algorithm_4_explanation}
    \end{subfigure}
    \hfill
    \begin{subfigure}{0.39\columnwidth}
        \centering
        \def\svgwidth{0.5\textwidth}
\newcommand{\robotidfontsize}{10pt}
\newcommand{\robotidvspace}{-1pt}
\begingroup%
  \makeatletter%
  \providecommand\color[2][]{%
    \errmessage{(Inkscape) Color is used for the text in Inkscape, but the package 'color.sty' is not loaded}%
    \renewcommand\color[2][]{}%
  }%
  \providecommand\transparent[1]{%
    \errmessage{(Inkscape) Transparency is used (non-zero) for the text in Inkscape, but the package 'transparent.sty' is not loaded}%
    \renewcommand\transparent[1]{}%
  }%
  \providecommand\rotatebox[2]{#2}%
  \newcommand*\fsize{\dimexpr\f@size pt\relax}%
  \newcommand*\lineheight[1]{\fontsize{\fsize}{#1\fsize}\selectfont}%
  \ifx\svgwidth\undefined%
    \setlength{\unitlength}{273.66612472bp}%
    \ifx\svgscale\undefined%
      \relax%
    \else%
      \setlength{\unitlength}{\unitlength * \real{\svgscale}}%
    \fi%
  \else%
    \setlength{\unitlength}{\svgwidth}%
  \fi%
  \global\let\svgwidth\undefined%
  \global\let\svgscale\undefined%
  \makeatother%
  \begin{picture}(1,1.37576413)%
    \lineheight{1}%
    \setlength\tabcolsep{0pt}%
    \put(0,0){\includegraphics[width=\unitlength,page=1]{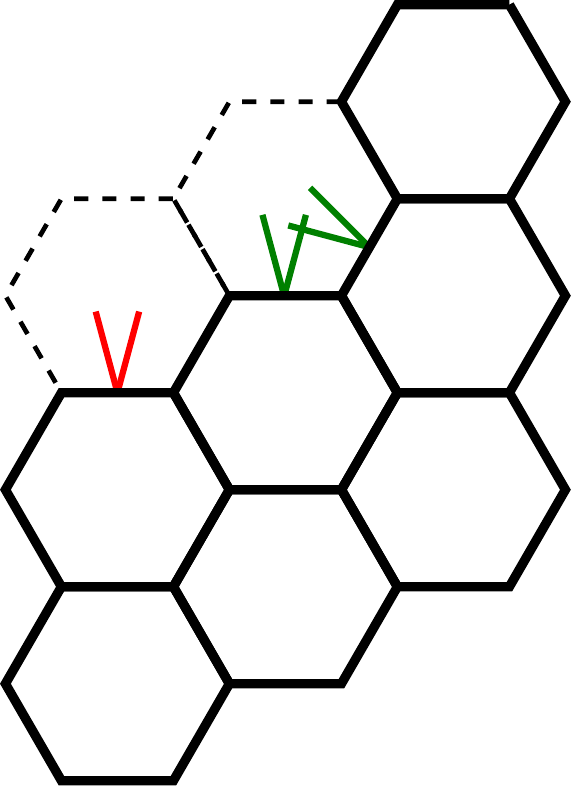}}%
    \put(0.50498362,0.67136562){\raisebox{\robotidvspace}{\color[rgb]{0,0,0}\makebox(0,0)[t]{\fontsize{\robotidfontsize}{0pt}\selectfont\smash{\begin{tabular}[t]{c}$j$\end{tabular}}}}}%
    \put(0.17563149,0.48095967){\raisebox{\robotidvspace}{\color[rgb]{0,0,0}\makebox(0,0)[lt]{\fontsize{\robotidfontsize}{0pt}\selectfont\smash{\begin{tabular}[t]{l}$i$\end{tabular}}}}}%
  \end{picture}%
\endgroup%

        \subcaption{}\label{fig:delay}
    \end{subfigure}
    \caption{(a) Notable situations during robot $i$'s determination of its role, where it lies on a different row from the midpoint $m = (-1, 3)$ of $Q^{(2)}_{-1}$. Dashed hexagons are target positions for $S$, with color overlays indicating membership to target segments. (i) Located on the midpoint's row, a robot occupying $(0,3)$---in the future---will be the nucleus (Algorithm~\ref{alg:determine_role_using_adjacent_col}, lines \ref{alg:determine_role_using_adjacent_col_case_1_if}--\ref{alg:determine_role_using_adjacent_col_case_1_endif}). (ii) Without a target $(0,3)$ position, the new attachment point is $m^\star = (-1,2)$, thus $(0,2)$ is the nucleus position (Algorithm~\ref{alg:determine_role_using_adjacent_col}, lines \ref{alg:determine_role_using_adjacent_col_case_2b_else}--\ref{alg:determine_role_using_adjacent_col_case_2b_endelse}). (iii) With its \wdfl{} wall between two null walls, robot $i$ is solely responsible for nucleating to the left (Algorithm~\ref{alg:determine_role_using_self}). (b) Suppose $g_i = g_j = -1$ and robot $i$ does not delay its signal (red) while robot $j$'s front wall is free. An unreachable position will occur if robot $i$'s front wall is occupied \emph{before} robot $j$'s.}
\end{figure}

\subsubsection{Delayed Wall Identification}
Next, using extrinsic information (neighbor $j$'s wall status vector $s_j(t)$), robot $i$ uses Algorithm~\ref{alg:delay_signal_activation} to delay certain fore-aft signals when a neighbor $j$'s fore-aft walls are still free. This prevents robot $i$ from attracting attachments before $j$, which could create an unreachable position (\fig{}\ref{fig:delay}). As a result of Algorithm~\ref{alg:delay_signal_activation}, robot $i$ obtains $\tilde{W}_i$, the set of walls to delay signal activation.


By the end of the \textit{Signal Activation} phase, robot $i$ activates signals on the walls in $W_i \backslash \tilde{W}_i$ for the current time step. As long as any of its walls remain free, robot $i$ stays in the \textit{Signal Activation} phase and repeats Algorithms~\ref{alg:identify_nominal_walls} and \ref{alg:delay_signal_activation}.

\begin{algorithm}
    \caption{Identify walls to delay signal activation}\label{alg:delay_signal_activation}
    \begin{algorithmic}[1]
        \State \textbf{Input:} Self-wall status vector $s_i$, set of $i$'s neighbors' wall status vectors (indexed by $i$'s walls) $\{s_{(w)} : w \in \{0, ..., 5 \} \}$, growth direction $g_i$
        \State \textbf{Output:} Set of walls requiring delayed activation $\tilde{W}_i$

        \State Initialize $\tilde{W}_i \gets \emptyset$

        \If{$g_i = -1$}
        \If{$s_{i,5} =$ occupied \textbf{and} $s_{(5),0} =$ free} \Comment{\wdfr{} neighbor has free \wdf{} wall}
        \State $\tilde{W}_i \gets \tilde{W}_i \cup \{0\}$ \Comment{delay \wdf{} wall}
        \EndIf

        \If{$s_{i,4} =$ occupied \textbf{and} $s_{(4),3} =$ free} \Comment{\wdrr{} neighbor has free \wdr{} wall}
        \State $\tilde{W}_i \gets \tilde{W}_i \cup \{3\}$ \Comment{delay \wdr{} wall}
        \EndIf

        \ElsIf{$g_i = 1$}
        \If{$s_{i,1} =$ occupied \textbf{and} $s_{(1),0} =$ free} \Comment{\wdfl{} neighbor has free \wdf{} wall}
        \State $\tilde{W}_i \gets \tilde{W}_i \cup \{0\}$ \Comment{delay \wdf{} wall}
        \EndIf

        \If{$s_{i,2} =$ occupied \textbf{and} $s_{(2),3} =$ free} \Comment{\wdrl{} neighbor has free \wdr{} wall}
        \State $\tilde{W}_i \gets \tilde{W}_i \cup \{3\}$ \Comment{delay \wdr{} wall}
        \EndIf
        \EndIf
        \State \Return $\tilde{W}_i$
    \end{algorithmic}
\end{algorithm}

\section{Theoretical Results}\label{sec:theoretical_results}
Here, we show that Huddle solves Problem~\ref{prob:problem_statement} by proving two properties separately: that $G(t)$ is reachable and always hole-free. Due to space constraints, we state the following lemmas without proof.

\begin{lemma}[Unique nucleation]\label{lem:unique_nucleation}
    Algorithms~\ref{alg:determine_role_using_self} and \ref{alg:determine_role_using_adjacent_col} ensure that any contiguous column segment $Q_{k+a}$ is connected to only one nucleus robot from an adjacent column $k$ (\textit{i.e.,} each node in the spanning tree has a unique parent).
\end{lemma}


\begin{corollary}\label{cor:segments_never_join}
    By Lemma~\ref{lem:unique_nucleation}, a segment $Q_{k+a}$ expands only from a single origin---a seeding robot $j$, $x_j = (k+a, q_j) \in Q_{k+a}$, attached to a nucleus robot in column $k$. Separate segments of column $k+a$---\textit{i.e.,} each having a seeding robot attached to a different nucleus robot in column $k$---will never expand to join one another.
\end{corollary}

Let a robot $i$ that has at least one free fore-aft wall be defined as a longitudinal frontier robot (LFR): $\text{free} \in \{s_{i,0}(t), s_{i,3}(t)\}$. Both robots $i$ and $j$ in \fig{}\ref{fig:delay} are LFRs. As such, an LFR exits Algorithm~\ref{alg:identify_nominal_walls} at line~\ref{alg:identify_nominal_walls_case1_endif} and can only attract attachments to its free fore-aft wall(s), upon which it ceases to be an LFR.

\begin{lemma}[Controlled longitudinal expansion]\label{lem:controlled_expansion}
    Suppose a robot $i$ with growth direction $g_i \neq 0$ in segment $Q_{p_i}$ has a neighbor $j$ belonging to segment $Q_{p_i-g_i}$, \textit{i.e.,} $|p_i - p_j| = 1$. If both are LFRs and share at least one free fore-aft wall index $w \in \{0, 3\}$ (\textit{i.e.}, $s_{i,w} = s_{j,w} = \text{free}$), Algorithm~\ref{alg:delay_signal_activation} ensures a passing-by robot $k$ joins robot $j$'s wall $w$ before robot $i$'s. That is, $p_k = p_j$ and $|q_k - q_j| = 1$.
\end{lemma}



\begin{corollary}\label{cor:no_unreachable_lfr}
    The controlled expansion guaranteed by Lemma~\ref{lem:controlled_expansion} ensures that robots attaching to LFRs do not cause $G(t)$ to be unreachable (\fig{}\ref{fig:delay}) or contain holes.
\end{corollary}

We now establish the main guarantees for Problem~\ref{prob:problem_statement}.

\begin{theorem}[Reachability]\label{thm:reachability}
    Following the algorithms in Section~\ref{sec:methodology}, the assembly satisfies Definition~\ref{def:reachability}.
\end{theorem}

\begin{proof}
    By Lemma~\ref{lem:unique_nucleation}, let segment $Q_{p_i}$ be the child of $Q_{p_j}$ in the spanning tree, \textit{i.e.,} seeded by one nucleus robot $j$ in $Q_{p_j}$. We prove reachability separately in longitudinal and lateral expansion.

    \begin{enumerate}
        \item \textit{Longitudinal}: From Corollary~\ref{cor:no_unreachable_lfr}, the longitudinal growth of segment $Q_{p_i}$ will not outpace its parent segment $Q_{p_j}$. The robot roles of the LFRs do not matter---without satisfying free fore-aft walls, flank signals will never be activated (Algorithm~\ref{alg:identify_nominal_walls}, lines \ref{alg:identify_nominal_walls_case1_if}--\ref{alg:identify_nominal_walls_case1_endif}). Thus, longitudinal expansion will result in reachable positions.

        \item \textit{Lateral}: A nucleus robot can only seed one segment per flank (Lemma~\ref{lem:unique_nucleation}). Suppose there exists a target segment $Q_{p_i}$ where $|Q_{p_i}| > 1$ and let robot $i$ be the first passing-by robot attracted to a flank wall of nucleus robot $j$, forming segment $Q_{p_i}$. Because robot $i$ is the first of its segment $Q_{p_i}$, no unreachable positions are created. On the other hand, a regular robot signals on a flank wall only when it is \emph{not} an LFR and is adjacent to an LFR (Algorithm~\ref{alg:identify_nominal_walls}, lines \ref{alg:identify_nominal_walls_case2_for}--\ref{alg:identify_nominal_walls_case2_endfor}). The presence of the adjacent LFR implies the longitudinal expansion case (for said LFR's segment), which is always reachable.
    \end{enumerate}\qed
\end{proof}

\begin{theorem}[Hole-free]
    Following the algorithms in Section~\ref{sec:methodology}, the assembly satisfies Definition~\ref{def:hole-free}.
\end{theorem}

\begin{proof}
    We show this using proof-by-contradiction. Recall that $S$ is assumed to be hole-free. From Corollary~\ref{cor:segments_never_join}, a column segment cannot be created from combining segments within its column---its longitudinal expansion must start from a seeding robot attached to a nucleus robot in the adjacent column.

    Suppose that $G(t)$ contains a hole---enclosed unoccupied positions in $S$---within column $k$. The hole is effectively a gap between two segments within the column; the segments must have originated from two seeding robots, each attached to a different nucleus robot in the adjacent column (Lemma~\ref{lem:unique_nucleation}). But this is a contradiction: $S$ is hole-free, so column $k$ must be seeded by only one nucleus robot in an adjacent column. Therefore, $G(t)$ is always hole-free.\qed
\end{proof}

\section{Experimental Results}\label{sec:experimental_results}

We evaluated our approach with Monte Carlo experiments (correctness) and a physics-based experiment in NVIDIA Isaac Sim (feasibility).

For our Monte Carlo experiments (\fig{}\ref{fig:monte_carlo_example}), robot motion is abstracted: passing-by robots instantaneously populate random openings signaled by in-assembly robots. We varied the number of simultaneous robot attachments from 1 to 4 and tested our approach on 40{,}000 randomized swarm shapes, with the largest consisting of more than 250 robots. Huddle successfully formed the desired assembly for all 160{,}000 trials.

\begin{figure}
    \centering
    \begin{subfigure}{0.325\columnwidth}
        \centering
        \begin{overpic}[width=0.65\textwidth]{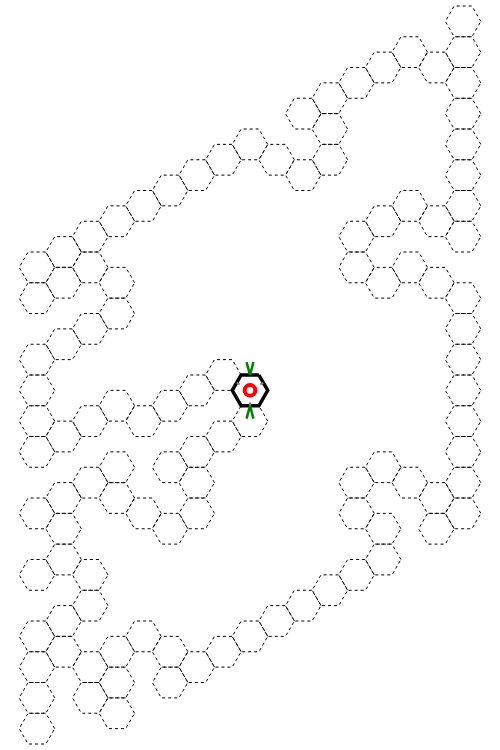}
            \put(-3,92){\small(a) Step 0}  
        \end{overpic}
        \phantomsubcaption{}\label{fig:monte_carlo_example_1}
    \end{subfigure}
    \hfill
    \begin{subfigure}{0.325\columnwidth}
        \centering
        \begin{overpic}[width=0.65\textwidth]{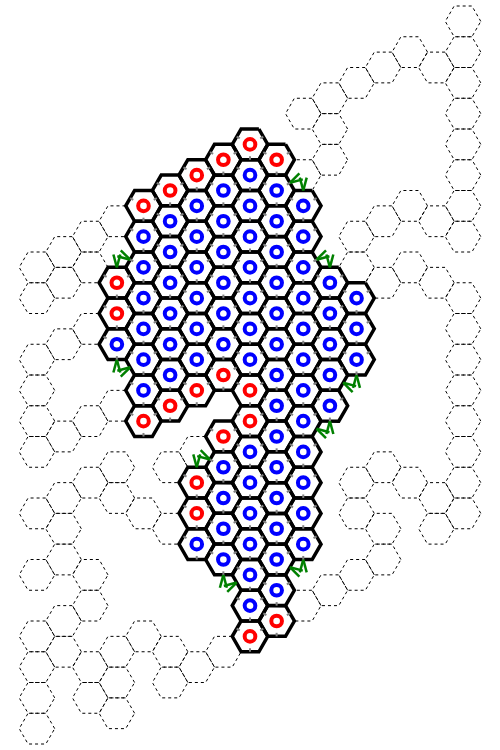}
            \put(-3,92){\small(b) Step 95}  
        \end{overpic}
        \phantomsubcaption{}\label{fig:monte_carlo_example_2}
    \end{subfigure}
    \hfill
    \begin{subfigure}{0.325\columnwidth}
        \centering
        \begin{overpic}[width=0.65\textwidth]{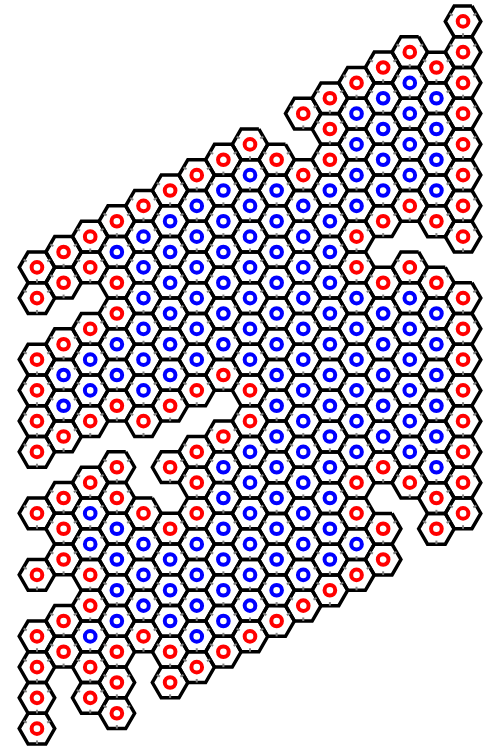}
            \put(-3,92){\small(c) Step 255}  
        \end{overpic}
        \phantomsubcaption{}\label{fig:monte_carlo_example_3}
    \end{subfigure}
    \caption{$256$-robot assembly example from the Monte Carlo experiments (one robot attachment per step). Dashed hexagons represent $X_b$, which are eventually populated by robots marked with red circles; hexagons marked with blue circles are interior robots. Note that our approach does not distinguish between perimeter and interior robots.}\label{fig:monte_carlo_example}
\end{figure}

In the physics-based validation, our hexagonal robot (0.18 \si{m} sides) uses IR transceivers for both signaling (up to 3 \si{m}) and neighbor communication, with simulated magnetic attachment. As Huddle is motion-agnostic, our robots are omnidirectional, moving at a top speed of 0.5 \si[per-mode=symbol]{\m\per\s} in a diffusive motion (\fig{}\ref{fig:motion_state_diagram}).

\begin{figure}
    \centering
    \begin{subfigure}{\textwidth}
        \centering
        \def\svgwidth{0.8\textwidth}
        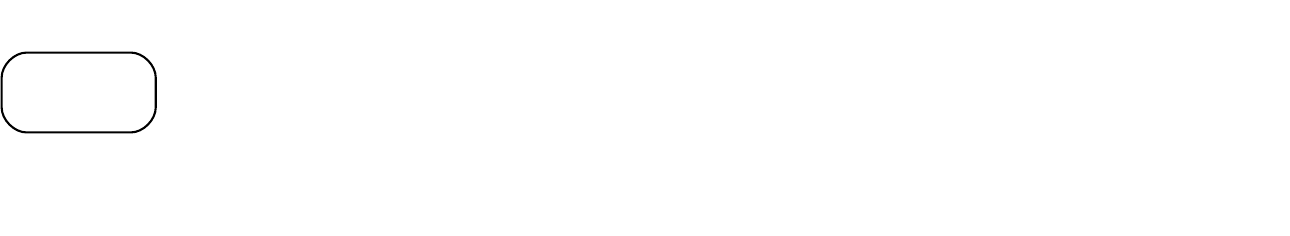
        \begin{overpic}[width=0.01\textwidth]{figures/blank.png}
            \put(0,0){\hspace{-305pt}\raisebox{40pt}{\small(a)}}  
        \end{overpic}
        \phantomsubcaption{}\label{fig:motion_state_diagram}
    \end{subfigure}
    \par \smallskip
    \begin{subfigure}{\textwidth}
        \centering
        \def\svgwidth{0.75\textwidth}
\newcommand{\labelfontsize}{8pt}
\newcommand{\labeltitlefontsize}{9pt}
\newcommand{\thresholdlabelhspace}{-15pt}
\begingroup%
  \makeatletter%
  \providecommand\color[2][]{%
    \errmessage{(Inkscape) Color is used for the text in Inkscape, but the package 'color.sty' is not loaded}%
    \renewcommand\color[2][]{}%
  }%
  \providecommand\transparent[1]{%
    \errmessage{(Inkscape) Transparency is used (non-zero) for the text in Inkscape, but the package 'transparent.sty' is not loaded}%
    \renewcommand\transparent[1]{}%
  }%
  \providecommand\rotatebox[2]{#2}%
  \newcommand*\fsize{\dimexpr\f@size pt\relax}%
  \newcommand*\lineheight[1]{\fontsize{\fsize}{#1\fsize}\selectfont}%
  \ifx\svgwidth\undefined%
    \setlength{\unitlength}{468.99648591bp}%
    \ifx\svgscale\undefined%
      \relax%
    \else%
      \setlength{\unitlength}{\unitlength * \real{\svgscale}}%
    \fi%
  \else%
    \setlength{\unitlength}{\svgwidth}%
  \fi%
  \global\let\svgwidth\undefined%
  \global\let\svgscale\undefined%
  \makeatother%
  \begin{picture}(1,0.30121268)%
    \lineheight{1}%
    \setlength\tabcolsep{0pt}%
    \put(0,0){\includegraphics[width=\unitlength,page=1]{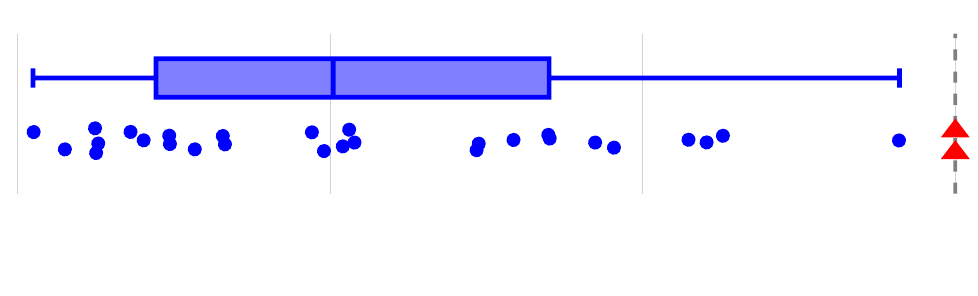}}%
    \put(0.01820416,0.07356131){\makebox(0,0)[t]{\fontsize{\labelfontsize}{0pt}\selectfont\lineheight{1.25}\smash{\begin{tabular}[t]{c}15\end{tabular}}}}%
    \put(0.33803598,0.07356131){\makebox(0,0)[t]{\fontsize{\labelfontsize}{0pt}\selectfont\lineheight{1.25}\smash{\begin{tabular}[t]{c}30\end{tabular}}}}%
    \put(0.65786779,0.07356131){\makebox(0,0)[t]{\fontsize{\labelfontsize}{0pt}\selectfont\lineheight{1.25}\smash{\begin{tabular}[t]{c}45\end{tabular}}}}%
    \put(0.97769961,0.07356131){\makebox(0,0)[t]{\fontsize{\labelfontsize}{0pt}\selectfont\lineheight{1.25}\smash{\begin{tabular}[t]{c}60\end{tabular}}}}%
    \put(0.50005149,0.01934983){\makebox(0,0)[t]{\fontsize{\labeltitlefontsize}{0pt}\selectfont\lineheight{1.25}\smash{\begin{tabular}[t]{c}Completion Time (min)\end{tabular}}}}%
    \put(0.93907942,0.27699568){\makebox(0,0)[t]{\lineheight{1.25}\smash{\begin{tabular}[t]{c}Timeout\end{tabular}}}}%
  \end{picture}%
\endgroup%

        \begin{overpic}[width=0.01\textwidth]{figures/blank.png}
            \put(0,0){\hspace{-285pt}\raisebox{66pt}{\small(b)}}  
        \end{overpic}
        \phantomsubcaption{}\label{fig:completion_time}
    \end{subfigure}
    \caption{(a) Robot diffusion motion state diagram. (b) Completion time of a 107-robot assembly forming ``$\mathcal{H}$'', where \nicefrac{2}{30} trials timed out (1 robot remaining).}
\end{figure}

Once a robot joins the assembly by attaching to $\geq 1$ in-assembly robot, it exchanges information with its neighbor(s) at every time step. Each robot $i$ transmits two elements: (1) its wall status vector $s_i(t)$ (encoded as a single 16-bit unsigned integer), and (2) its hexagonal coordinates $x_i = (p_i, q_i)$ combined with the unit vector $\hat{u}_{i,w}$ of the transmitting wall $w$ (jointly encoded as a 32-bit unsigned integer). Communication only occurs between in-assembly robots (as noted in earlier sections).

Due to continuous-space motion, the orientation of a joining robot is unlikely to match the attracting robot's. To that end, joining robots transform their wall directions to a \textit{north-front} orientation (\fig{}\ref{fig:single_hex}) using the attracting robot's $\hat{u}_w$, ensuring Huddle works as described in Section~\ref{sec:methodology}.

We conducted an experiment with 107 robots in a square arena of 179 \si{\m^2} to form a highly non-convex shape: the calligraphic letter ``$\mathcal{H}$'' (\fig{}\ref{fig:isaac_sim_screenshots}). Over 30 trials, only two did not complete the shape within a simulated hour, timing out with 106 in-assembly robots (\fig{}\ref{fig:completion_time}). While this shows a high success rate ($>99\%$ of the shape is assembled in all trials), more efficient motion strategies can accelerate the assembly's completion, as we discuss in Section~\ref{sec:discussion}. We include a video demonstrating complete shape assembly in the supplementary materials.

\begin{figure*}
    \centering
    \begin{subfigure}{0.328\linewidth}
        \centering
        \includegraphics[width=\linewidth]{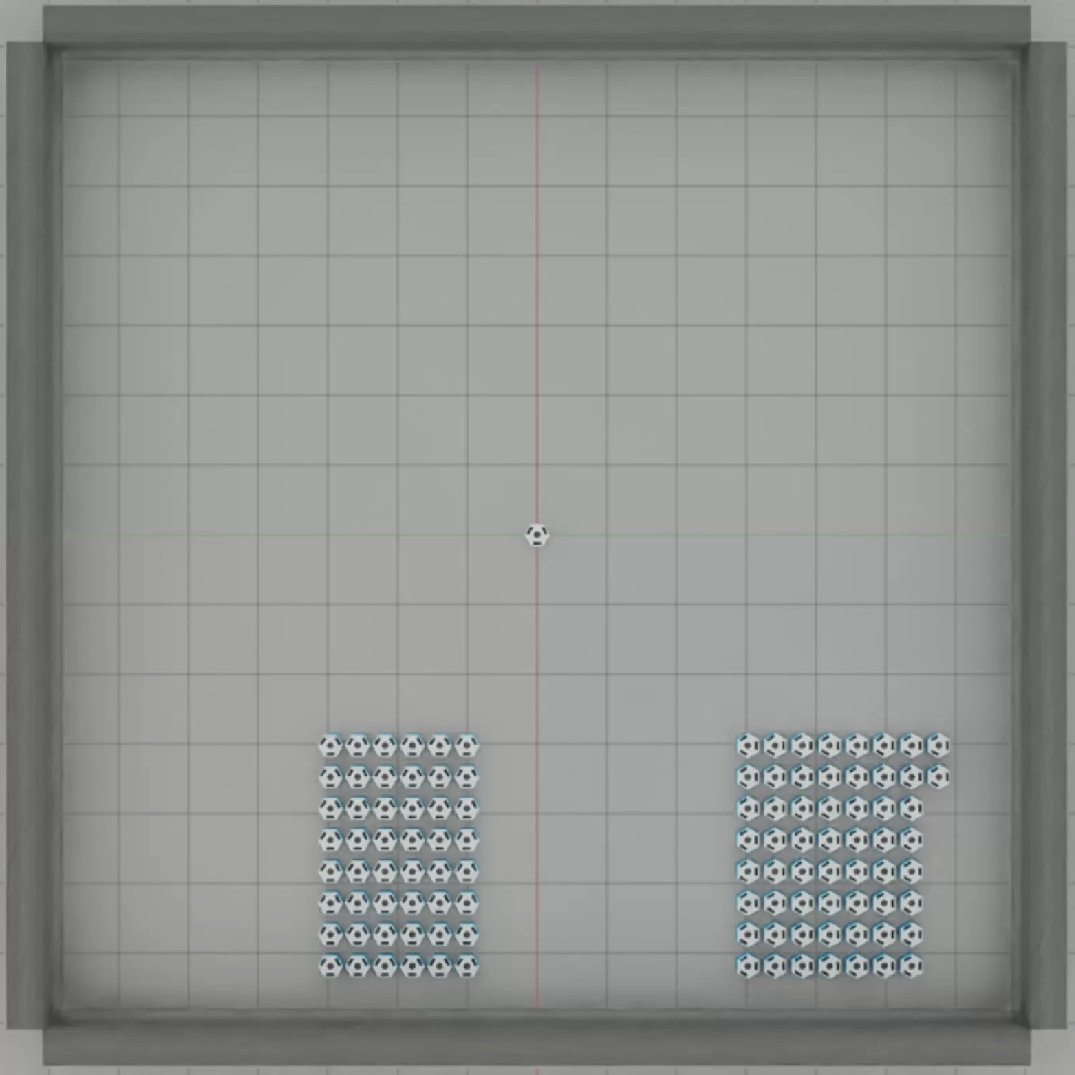}
        \subcaption{}
    \end{subfigure}
    \hfill
    \begin{subfigure}{0.328\linewidth}
        \centering
        \includegraphics[width=\linewidth]{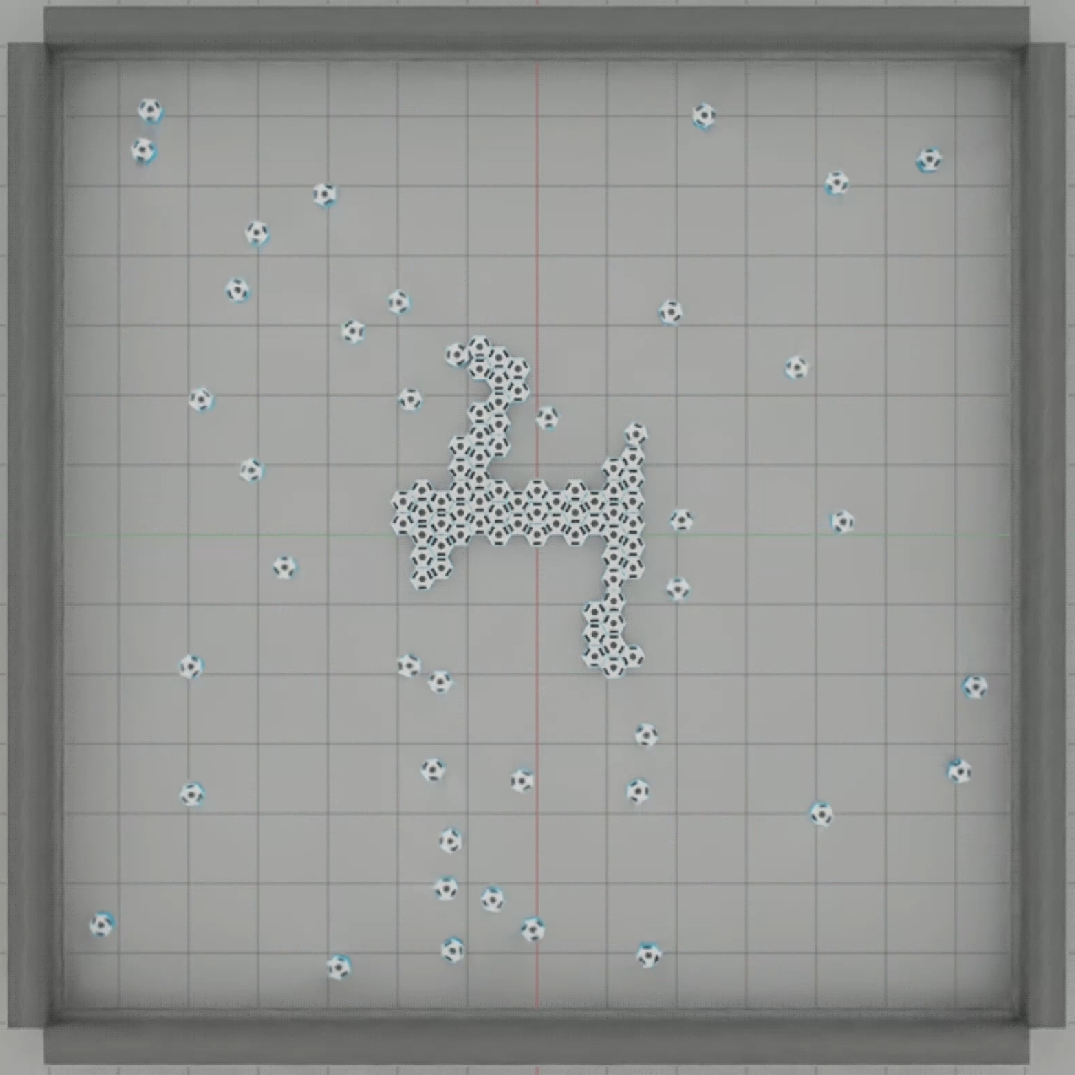}
        \subcaption{}
    \end{subfigure}
    \hfill
    \begin{subfigure}{0.328\linewidth}
        \centering
        \includegraphics[width=\linewidth]{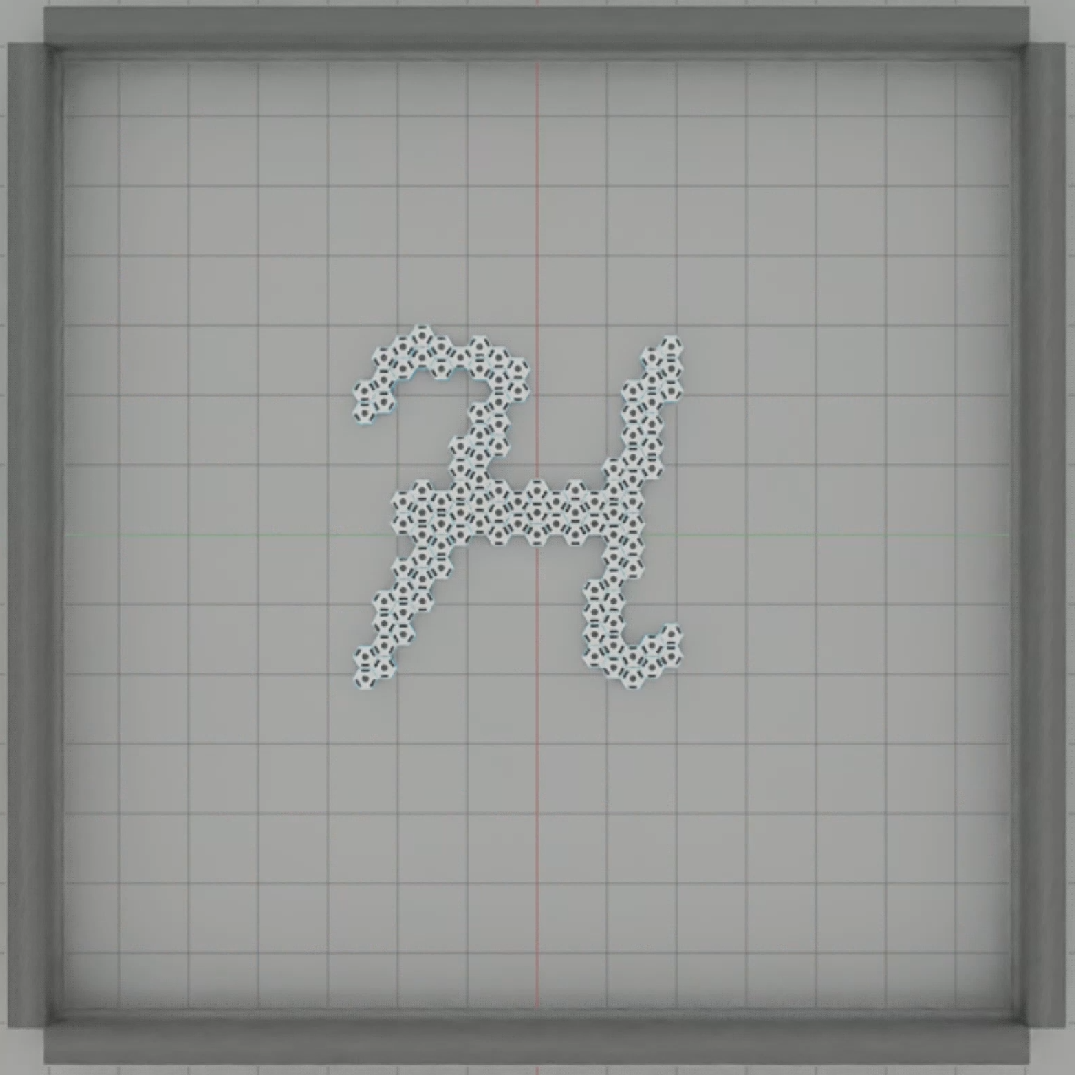}
        \subcaption{}
    \end{subfigure}
    \caption{107 robots forming ``$\mathcal{H}$''. (a) We place the root robot at the world origin, while the rest of the robots start in two clusters. (b) The robots diffuse within an arena bounded by four walls, joining the assembly as attracted by signals. (c) Each robot eventually finds an opening to attach, completing the assembly.}\label{fig:isaac_sim_screenshots}
\end{figure*}
\section{Discussion}\label{sec:discussion}

Using only the shape perimeter, Huddle enables the parallel, asynchronous assembly of arbitrary shapes by robots with continuous motion. The robots do not need to be identified or require pre-determined starting locations or goal assignments. We make no assumptions about the robots' motion, allowing order-independent assembly as long as they can encounter signals---intentionally or by chance. Once part of the assembly, each robot communicates limited information only with their nearest neighbors.

While pose localization is not required, the assembly completion rate can improve by using robots with more efficient motion control strategies that may rely on localization. This allows for robot navigation that can either take a decentralized form---such as circling the assembly (similar to~\cite{werfelDistributedConstructionMobile2006})---or a more centralized one---such as clustering locations (near assembly openings) computed and disseminated by a server.

That said, there are practical limitations to using Huddle. In specifying the shape, the designer must be careful to avoid shapes that can cause signal occlusion during assembly; shapes containing one-robot-wide gaps tend to cause this---especially if those gaps are long continuous channels (\fig{}\ref{fig:monte_carlo_example_3}). Huddle also does not account for physical effects of signals which, depending on the medium, may be significant. For instance, unlike LED signals, two IR signals targeting the same opening may interfere (\fig{}\ref{fig:delay}).

Though Huddle is designed with hexagonal units in mind, the approach extends to square units, another shape that admits monohedral tessellation (tiling by congruent copies). To do this, we adjust Definition~\ref{def:reachability} to allow for at most 2 simultaneous connections. The four-wall structure simplifies flank signaling logic, as each lateral direction corresponds to a single wall rather than two. However, Huddle does not extend trivially to triangular units, the remaining regular polygon admitting monohedral tessellation.

Our formulation using hexagons lends itself to circular robots as well, since they can each have at most 6 neighbors (2-D kissing number). If rules about robot orientation can be enforced, Huddle can be extended to shape formation applications, like in \cite{rubensteinProgrammableSelfassemblyThousandrobot2014,yangParallelShapeFormation2022,liDecentralizedProgressiveShape2019}, where a robot's circular footprint can either be its physical shape or safety boundary.
\section{Conclusion}\label{sec:conclusion}
We presented Huddle, a parallel shape assembly algorithm that enables decentralized, minimalistic robots to form arbitrary shapes. With only the perimeter as input, Huddle guides assembly expansion without bespoke motion strategies, relying instead on signaling mechanisms alone. We proved that Huddle activates signals in a controlled manner such that the assembly will never contain unreachable openings or gaps. We also demonstrated its effectiveness with a high success rate in forming a highly non-convex shape despite using robots with unsophisticated motion. As future work, we aim to validate Huddle on real robotic platforms and explore avenues to reconfigure the assembly.

\bibliographystyle{unsrtnat}
\bibliography{references}

\end{document}